\documentclass{llncs}
\usepackage{amsmath}
\usepackage{amssymb,amsfonts}
\usepackage{times}
\usepackage{indentfirst}
\usepackage{graphicx}

\begin{document}
\title{Rough sets and matroidal contraction}         

\author{Jingqian Wang, William Zhu~\thanks{Corresponding author.
E-mail: williamfengzhu@gmail.com (William Zhu)}}

\institute{
Lab of Granular Computing, \\
Zhangzhou Normal University, Zhangzhou, China}



\date{\today}          
\maketitle

\large
\begin{abstract}
Rough sets are efficient for data pre-processing in data mining.
As a generalization of the linear independence in vector spaces, matroids provide well-established platforms for greedy algorithms.
In this paper, we apply rough sets to matroids and study the contraction of the dual of the corresponding matroid.
First, for an equivalence relation on a universe, a matroidal structure of the rough set is established through the lower approximation operator.
Second, the dual of the matroid and its properties such as independent sets, bases and rank function are investigated.
Finally, the relationships between the contraction of the dual matroid to the complement of a single point set and the contraction of the dual matroid to the complement of the equivalence class of this point are studied.

\end{abstract}


\section{Introduction}
\label{section0}
The theory of rough sets was proposed by Pawlak~\cite{Pawlak91Rough,Pawlak82Rough} in 1982 as a tool to conceptualize, organize and analyze various types of data in data mining.
This theory is especially useful for dealing with uncertain and vague knowledge in information systems.
Using the concepts of lower and upper approximations in rough set theory, knowledge hidden in information systems may be unraveled and expressed in the form of decision rules.
The basic operators in rough set theory are approximation operators.
Many examples of applications of the rough
set theory to process control, economics, medical diagnosis, biochemistry, environmental science, biology, chemistry psychology, conflict analysis and other fields can be found in \cite{Liu06TheTransitive,LiuWang06Semantic,HungSankarSkowron11Rough,QinPeiDu05TheRelationship,WangLiu00TheInconsistency,WuZhang06Rough,YaoChen05Subsystem,ZhongYaoOhshima03peculiarity,ZhuWang06Covering}.

Matroids~\cite{Lai01Matroid,Li07Some,Mao06TheRelation,WangZhu11Matroidal,ZhuWang11Matroidal} are structures that generalize linear independence in vector spaces.
They borrow extensively linear algebra and graph theory, and have a variety of applications in combinatorial optimization, integer programming, secret communication, and so on.
Recently, the combination of rough sets and matroids has been studied in \cite{LiLiu12Matroidal,WangZhu11Matroidal,WangZhuZhu10Abstract,ZhuWang11Matroidal,ZhuWang11Rough}.

In this paper, a matroidal structure of rough sets is constructed, and then the contraction of the dual of the matroid is studied.
First, a matroid is induced by the lower approximation operator in rough sets.
Second, the dual of the matroid is investigated.
Finally, the relationships between the contraction of the dual matroid to the complement of a single point set and the contraction of the dual matroid to the complement of the equivalence class of this point are studied.

The rest of this paper is organized as follows. Section~\ref{section1} reviews some fundamental concepts about rough sets and matroids.
In Section~\ref{section2}, a matroid is induced by the lower approximation operator in rough sets, and the relationships between the contraction of the dual matroid to the complement of a single point set and the contraction of the dual matroid to the complement of the equivalence class of this point are studied.
Finally, this paper is concluded and further work is indicated in Section~\ref{section3}.
\section{Basic definitions}
\label{section1}

This section recalls some fundamental definitions related to rough sets and matroids.

\begin{definition}(Approximation space~\cite{Yao98Relational,Yao98Constructive})
\label{definition1}
Let $U$ be a nonempty and finite set called universe and $R$ an equivalence relation on $U$.
The ordered pair $(U,R)$ is called a Pawlak's approximation space.
\end{definition}

In rough sets, a pair of approximation operators are used to describe an object. In the following definition, a widely used pair of
approximation operators are introduced.

\begin{definition}(Approximation operator~\cite{Yao98Relational,Yao98Constructive})
\label{definition2}
Let $R$ be an equivalence relation on $U$.
A pair of approximation operators $R_{*}$, $R^{*}:2^{U}\rightarrow 2^{U}$, are defined as follows: for all $X\subseteq U$,\\
$R_{*}(X)=\{x\in U:RN(x)\subseteq X\}$.\\
$R^{*}(X)=\{x\in U:RN(x)\bigcap X\neq \emptyset\}$.\\
Where $RN(x)=\{y\in U:xRy\}$ is called the equivalence class of $x$, and $RN(x)$ is a member of the partition generated 
by $R$ on $U$.
They are called the lower and upper approximation operators with respect to $R$, respectively.
\end{definition}

Let $\emptyset$ be the empty set, $X^{c}$ the complement of $X$ in $U$, from the definition of approximation operators, we have the following conclusions about them.

\begin{proposition}(~\cite{Yao98Relational,Yao98Constructive})
\label{proposition12}
The properties of the Pawlak's rough sets:\\
(1L) $ R_{*}(U)= U $~~~~~~~~~~~~~~~~~~~~~~~~~~~~~~~~~~(1H) $ R^{*}(U)= U $\\
(2L) $ R_{*}(\phi)= \phi$ ~~~~~~~~~~~~~~~~~~~~~~~~~~~~~~~~~~(2H) $ R^{*}(\phi)= \phi$\\
(3L) $ R_{*}(X)\subseteq X $~~~~~~~~~~~~~~~~~~~~~~~~~~~~~~~~(3H) $X \subseteq R^{*}(X) $\\
(4L) $R_{*}(X \bigcap Y )= R_{*}(X) \bigcap R_{*}(Y )$\\
(4H) $R^{*}(X \bigcup Y )= R^{*}(X) \bigcup R^{*}(Y )$\\
(5L) $R_{*}(R_{*}(X))= R_{*}(X) $~~~~~~~~~~~~~~~~~(5H) $R^{*}(R^{*}(X))= R^{*}(X) $\\
(6L) $X \subseteq Y \Rightarrow R_{*}(X) \subseteq R_{*}(Y )$~~~~~~~(6H) $X \subseteq Y \Rightarrow R^{*}(X) \subseteq R^{*}(Y )$\\
(7L) $R_{*}((R_{*}(X))^{c})= (R_{*}(X))^{c}$~~~~~~~~~(7H) $R^{*}((R^{*}(X))^{c})= (R^{*}(X))^{c}$\\
(8LH) $R_{*}(X^{c})= (R^{*}(X))^{c} $\\
(9LH) $R_{*}(X) \subseteq R^{*}(X)$ 
\end{proposition}
  
Matroids are algebraic structures that generalize linear independence in vector spaces.
They have a variety of applications in integer programming, combinatorial optimization, algorithm design, and so on.
In the following definition, one of the most valuable definitions of matroids is presented in terms of independent sets.

\begin{definition}(Matroid~\cite{Lai01Matroid})
\label{definition3}
A matroid is an ordered pair $M=(U,\textbf{I})$ where $U$(the ground set) is a finite set, and $\textbf{I}$ (the independent sets) a family of subsets of $U$ with the following properties:\\
(I1) $\emptyset \in \textbf{I}$.\\
(I2) If $I\in \textbf{I}$, and $I^{'}\subseteq I$, then $I^{'}\in \textbf{I}$.\\
(I3) If $I_{1}, I_{2}\in \textbf{I}$, and $|I_{1}|< |I_{2}|$, then there exists $e\in I_{2}-I_{1}$, such that $I_{1}\bigcup \{e\}\in \textbf{I}$, where $|I|$ denotes the cardinality of $I$.
\end{definition}

\begin{example}
\label{example1}
Let $U=\{a,b,c,d\}$, $\textbf{I}=\{\emptyset,\{a\},\{b\},\{c\},\{a,b\},\{a,c\},\{b,$ $c\}\}$. Then $(U,\textbf{I})$ is a matroid.
\end{example}

If a subset of the ground set is not an independent set of a matroid, then it is called a dependent set of the matroid.
Based on the dependent set, we introduce the circuit of a matroid.
For this purpose, two denotations are presented.

\begin{definition}(\cite{Lai01Matroid})
\label{definition4}
Let $U$ be a nonempty and finite set and $\textbf{A}\subseteq 2^{U}$ a family of subsets of $U$. One can denote:\\
$Max(\textbf{A})=\{X\in \textbf{A}:\forall Y\in \textbf{A},X\subseteq Y\Rightarrow X=Y\}$.\\
$Min(\textbf{A})=\{X\in \textbf{A}:\forall Y\in \textbf{A},Y\subseteq X\Rightarrow X=Y\}$.
\end{definition}

The dependent set of a matroid generalizes the linear dependence in vector spaces and the cycle in graphs.
Any circuit of a matroid is a minimal dependent set.

\begin{definition}(Circuit~\cite{Lai01Matroid})
\label{definition5}
Let $M=(U,\textbf{I})$ be a matroid.
A minimal dependent set in $M$ is called a circuit of $M$, and we denote the family of all circuits of $M$ by $\textbf{C}(M)$,
i.e., $\textbf{C}(M)=Min(\textbf{I}^{c})$, where $\textbf{I}^{c}=2^{U}-\textbf{I}$.
\end{definition}

\begin{example}(Continued from Example~\ref{example1})
$\textbf{C}(M)=\{\{d\},\{a,b,c\}\}.$
\end{example}

A base of a matroid is a maximal independent set.

\begin{definition}(Base~\cite{Lai01Matroid})
\label{definition6}
Let $M=(U,\textbf{I})$ be a matroid.
A maximal independent set in $M$ is called a base of $M$, and we denote the family of all bases of $M$ by $\textbf{B}(M)$,
i.e., $\textbf{B}(M)=Max(\textbf{I})$.
\end{definition}

\begin{example}(Continued from Example~\ref{example1})
$\textbf{B}(M)=\{\{a,b\},\{a,$ $c\},\{b,c\}\}$.
\end{example}

The following proposition shows that a matroid can be determined by bases.

\begin{proposition}(Base axiom~\cite{Lai01Matroid})
\label{proposition0}
Let $U$ be a finite set and $\textbf{B}$ a family of subsets of $U$.
Then there exists $M=(U,\textbf{I})$ such that $\textbf{B}=\textbf{B}(M)$ iff $\textbf{B}$ satisfies the following two conditions:\\
(B1) $\textbf{B}\neq \emptyset$.\\
(B2) If $B_{1}$, $B_{2}\in \textbf{B}$ and $x\in B_{1}-B_{2}$, then there exists an element $y\in B_{2}-B_{1}$ such that $(B_{1}-\{x\})\bigcup \{y\}\in \textbf{B}$.
\end{proposition}

The dimension of a vector space and the rank of a matrix are useful concepts in linear algebra. The rank function of a matroid is a generalization of these two concepts.

\begin{definition}(Rank function~\cite{Lai01Matroid})
\label{definition7}
Let $M=(U,\textbf{I})$ be a matroid. The rank function $r_{M}$ of $M$ is defined as $r_{M}(X)=max\{|I|:I\subseteq X,I\in \textbf{I}\}$ for all $X\subseteq U$. $r_{M}(X)$ is called the rank of $X$ in $M$.
\end{definition}

\begin{example}(Continued from Example~\ref{example1})
Suppose $X_{1}=\{a\}$, $X_{2}=\{a,b,c\}$, $X_{3}=\{a,b,d\}$. Then $r_{M}(X_{1})=1$, $r_{M}(X_{2})=2$, $r_{M}(X_{3})=2 $.
\end{example}

The following proposition presents the properties of the rank function of a matroid.

\begin{proposition}(\cite{Lai01Matroid})
\label{proposition1}
Let $M=(U,\textbf{I})$ be a matroid and $r_{M}$ the rank function of $M$. Then $r_{M}$ satisfies the following conditions:\\
(R1) If $X\in 2^{U}$, then $0\leq r_{M}(X)\leq |X|$.\\
(R2) If $X\subseteq Y\subseteq U$, then $r_{M}(X)\leq r_{M}(Y)$.\\
(R3) If $X, Y\subseteq U$, then $r_{M}(X)+ r_{M}(Y)\geq r_{M}(X\bigcup Y)+r_{M}(X\bigcap Y)$.
\end{proposition}

Given a matroid, we can generate a new matroid through the following proposition.

\begin{proposition}(\cite{Lai01Matroid})
\label{proposition2}
Let $M=(U,\textbf{I})$ be a matroid and $\textbf{B}^{*}=\{U-B:B\in \textbf{B}(M)\}$. Then $\textbf{B}^{*}$ is the family of bases of a matroid on $U$.
\end{proposition}

The new matroid in the above proposition, whose ground set is $U$ and whose set of bases is $\textbf{B}^{*}$, is called the dual of $M$ and denoted by $M^{*}$.
Generally, the independent sets, the circuits, the bases and the rank function of $M^{*}$ are called the coindependent sets, the cocircuits, the cobases and the corank function of $M$.

\section{Rough sets and matroidal contraction}
\label{section2}
In this section, we firstly establish a matroidal structure of rough sets by the lower approximation operator.
The dual of the matroid and its properties such as independent sets, bases and rank function are provided.
Finally, the relationships between the contraction of the dual matroid to the complement of a single point set and the contraction of the dual matroid to the complement of the equivalence class of this point are studied.\\

In the following proposition, for an equivalence relation on a universe, we establish a family of subsets through the lower approximation operator, and prove it satisfies the independent set axiom of a matroid.
In other words, it determines a matroid.

\begin{proposition}
\label{proposition3}
Let $R$ be an equivalence relation on $U$.
Then $\textbf{I(}R)=\{X\subseteq U:R_{*}(X)=\emptyset \}$ satisfies (I1), (I2) and (I3) of Definition~\ref{definition3}.
\end{proposition}

\begin{proof}
First, according to Proposition~\ref{proposition12}, $R_{*}(\emptyset)=\emptyset $. Then $ \emptyset \in \textbf{I}(R)$.\\
Second, let $I\in \textbf{I}(R)$, $I^{'}\subseteq I$.
Since $I\in \textbf{I}(R)$, so $R_{*}(I)=\emptyset$.
According to Proposition~\ref{proposition12}, $R_{*}(I^{'})\subseteq R_{*}(I)=\emptyset$. Therefore, $R_{*}(I^{'})=\emptyset $, i.e., $I^{'}\in \textbf{I}(R)$.\\
Third, let the partition generated by $R$ on $U$ be $U/R=\{P_{1},P_{2},\cdots,P_{m}\}$. 
Let $I_{1}$, $I_{2}\in \textbf{I}(R)$ and $|I_{1}|<|I_{2}|$. 
Since $I_{1}=I_{1}\bigcap U$ and $I_{2}=I_{2}\bigcap U$, so $I_{1}=I_{1}\bigcap (\bigcup \limits_{i=1}^{m}P_{i})=\bigcup \limits_{i=1}^{m}(I_{1}\bigcap P_{i})$
and $I_{2}=I_{2}\bigcap (\bigcup \limits_{i=1}^{m}P_{i})=\bigcup \limits_{i=1}^{m}(I_{2}\bigcap P_{i})$.
Since $I_{1}$, $I_{2}\in \textbf{I}(R)$, so $R_{*}(I_{1})=\emptyset$ and $R_{*}(I_{2})=\emptyset$.
Thus, for all $1\leq i\leq m$, $(I_{1}\bigcap P_{i})\subset P_{i}$ and $(I_{2}\bigcap P_{i})\subset P_{i}$.
Since $|I_{1}|<|I_{2}|$, so $|\bigcup \limits_{i=1}^{m}(I_{1}\bigcap P_{i})|<|\bigcup \limits_{i=1}^{m}(I_{2}\bigcap P_{i})|$,
i.e., $\sum\limits_{i=1}^{m}|I_{1}\bigcap P_{i}|< \sum\limits_{i=1}^{m}|I_{2}\bigcap P_{i}|$.
Therefore, there exists $1\leq i\leq m$ such that $|I_{1}\bigcap P_{i}|<|I_{2}\bigcap P_{i}|$ 
(In fact, if for all $1\leq i\leq m$ such that $|I_{1}\bigcap P_{i}|\geq |I_{2}\bigcap P_{i}|$, then $\sum\limits_{i=1}^{m}|I_{1}\bigcap P_{i}|\geq \sum\limits_{i=1}^{m}|I_{2}\bigcap P_{i}|$, i.e., $|I_{1}|\geq |I_{2}|$. It is contradictory with $|I_{1}|<|I_{2}|$).
Thus, $|I_{1}\bigcap P_{i}|<|I_{2}\bigcap P_{i}|<|P_{i}|$,  and there exists $e\in (I_{2}\bigcap P_{i})-(I_{1}\bigcap P_{i})\subseteq I_{2}-I_{1}$ such that $(I_{1}\bigcap P_{i})\bigcup \{e\}\subset P_{i}$, i.e., $R_{*}(I_{1}\bigcup \{e\})=\emptyset$. Hence $I_{1}\bigcup \{e\}\in \textbf{I}(R)$.
This completes the proof.

In conclusion, $\textbf{I}(R)$ satisfies (I1), (I2) and (I3) of Definition~\ref{definition3}. Therefore, there exists a matroid on $U$ such that $\textbf{I}(R)$ is the family of its independent sets, and the matroid is denoted by $M(R)=(U,\textbf{I}(R))$. The family of bases of $M(R)$ denoted by $\textbf{B}(R)$.
\end{proof}

\begin{example}
\label{example2}
Let $U=\{a,b,c,d,e\}$, $R$ an equivalence relation on $U$ and the partition generated by $R$ on $U$ be $U/R=\{\{a,b\},\{c,d,e\}\}$. 
According to Proposition~\ref{proposition3}, $\textbf{I}(R)$
$=\{\emptyset,\{a\},\{b\},\{c\},\{d\},\{e\},\{a,c\},\{a,d\},\{a,$ $e\},\{b,c\},\{b,d\},\{b,e\},\{c,d\},\{c,e\},\{d,e\},\{a,c,d\},\{a,c,e\},\{a,d,e\},$ $\{b,c,d\},\{b,c,e\},\{b,d,e\}\}.$ Therefore, the matroid induced by the lower approximation operator is $M(R)=(U,\textbf{I}(R))$.
\end{example}

The following two corollaries represent the independent sets and the bases of the matroid induced by the lower approximation operator.

\begin{corollary}
\label{corollary1}
Let $R$ be an equivalence relation on $U$ and $M(R)$ the induced matroid. Then $\textbf{I}(R)=\{X\subseteq U:R_{*}(X)=\emptyset\}=\{X\subseteq U:\forall x\in U,RN(x)$
$\nsubseteq X \}.$
\end{corollary}

\begin{corollary}
\label{corollary2}
Let $R$ be an equivalence relation on $U$, $M(R)$ the induced matroid. Then $\textbf{B}(R)=Max(\textbf{I}(R))=\{X\subseteq U:\forall x\in U,|RN(x)$ $\bigcap X|=|RN(x)|-1\}.$
\end{corollary}

\begin{example}(Continued from Example~\ref{example2})
\label{example3}
$\textbf{B}(R)=\{\{a,c,d\},\{a,c,e\},$ $\{a,d,e\},\{b,c,d\},\{b,c,e\},\{b,d,e\}\}$.
\end{example}

The dual of the matroid and its properties such as independent sets, bases and rank function are investigated in Proposition~\ref{proposition4},
Corollary~\ref{corollary3} and Corollary~\ref{corollary4}.

\begin{proposition}
\label{proposition4}
Let $R$ be an equivalence relation on $U$, $M(R)$ the induced matroid and $\textbf{B}^{*}(R)=\{U-B:B\in \textbf{B}(R)\}$$=\{X\subseteq $ $U:\forall x\in U,|RN(x)\bigcap X|=1\}$. Then $\textbf{B}^{*}(R)$ satisfies the base axiom.
\end{proposition}

\begin{proof}
According to Proposition~\ref{proposition2}, it is straightforward.
\end{proof}

\begin{example}(Continued from Example~\ref{example2})
\label{example4}
$\textbf{B}^{*}(R)=\{\{a,c\},\{a,d\},\{a,e\},$ $\{b,c\},\{b,d\},\{b,e\}\}$.
\end{example}

According to Proposition~\ref{proposition4}, $\textbf{B}^{*}(R)$ is the family of bases of a matroid, and the matroid denoted by $M^{*}(R)$.
In fact, we know $M^{*}(R)$ is the dual matroid of $M(R)$, and we denote the family of the independent sets and the rank function of $M^{*}(R)$
as $\textbf{I}^{*}(R)$ and $r_{M^{*}(R)}$, respectively.

\begin{corollary}
\label{corollary3}
Let $R$ be an equivalence relation on $U$ and $M^{*}(R)$ the dual of the induced matroid $M(R)$.
Then $\textbf{I}^{*}(R)=Min(\textbf{B}^{*}(R))=\{X\subseteq $ $U:\forall x\in U,|RN(x)\bigcap$ $ X|\leq1\}.$
\end{corollary}

We denote the dual matroid as $M^{*}(R)=(U,I^{*}(R))$.

\begin{example}(Continued from Example~\ref{example2})
$\textbf{I}^{*}(R)=\{\emptyset,\{a\},\{b\},\{c\},\{d\},$ $\{e\},\{a,c\},\{a,d\},\{a,e\},\{b,c\},\{b,d\},\{b,e\}\}$. Therefore the dual of the matroid $M(R)$ is $M^{*}(R)=(U,I^{*}(R))$.
\end{example}

\begin{corollary}
\label{corollary4}
Let $R$ be an equivalence relation on $U$ and $M^{*}(R)$ the dual of the induced matroid $M(R)$.
For all $x\in U$ and $X\subseteq U$, $r_{M^{*}(R)}(X)=|\{RN(x):RN(x)\bigcap X\neq \emptyset\}|.$
\end{corollary}

\begin{example}(Continued from Example~\ref{example2})
Suppose $X_{1}=\{a,c\}$ and $X_{2}=\{c,d,e\}$. Then $r_{M^{*}(R)}(X_{1})=|\{\{a,b\},\{c,d,e\}\}|=2$ and $r_{M^{*}(R)}(X_{2})=|\{\{c,d,e\}\}|=1$.
\end{example}

The following proposition shows a relationship between the ranks of two subsets of a universe.

\begin{proposition}(\cite{Lai01Matroid})
\label{proposition5}
Let $U$ be a finite set and $r:2^{U}\rightarrow Z$ a function satisfying conditions (R2) and (R3) of Proposition~\ref{proposition1}.
If $X,Y\subseteq U$ such that for all $y\in Y-X$, $r(X)=r(X\bigcup \{y\})$, then $r(X)=r(X\bigcup Y)$.
\end{proposition}

The following two definitions show two special matroids.

\begin{definition}(Restriction~\cite{Lai01Matroid})
\label{definition8}
Let $M=(U,\textbf{I})$ be a matroid and $X\subseteq U$. Then $M|X=(X,\textbf{I}_{X})$ is a
matroid called the restriction of $M$ to $X$, where $\textbf{I}_{X}=\{I\subseteq X:I\in \textbf{I}\}$.
\end{definition}

\begin{definition}(Contraction~\cite{Lai01Matroid})
\label{definition9}
Let $M=(U,\textbf{I})$ be a matroid, $T\subseteq U$ and $B_{T}$ be a base of $M|T$, i.e., $B_{T}\in \textbf{B}(M|T)$. Then $M/T=(U-T,\textbf{I}^{'})$ is a matroid called
the contraction of $M$ to $U-T$, where $\textbf{I}^{'}=\{I\subseteq U-T:I\bigcup B_{T}\in \textbf{I}\}$. (The definition of $M/T$ has no relationship with the selection of $B_{T}\in \textbf{B}(M|T)$)
\end{definition}

The following proposition shows a relationship of ranks between a matroid and the restriction of the matroid to a subset.

\begin{proposition}(\cite{Lai01Matroid})
\label{proposition6}
Let $M=(U,\textbf{I})$ be a matroid and $T\subseteq U$.
For all $X\subseteq U-T$, $r_{M/T}(X)=r_{M}(X\bigcup T)-r_{M}(T)$.
\end{proposition}

In the following four propositions, we investigate the relationships between some characteristics of the contraction of the dual matroid to the complement of a single point set and to the complement of the equivalence class of this point, respectively, such as independent sets, bases, rank functions and circuits.

\begin{proposition}
\label{proposition7}
Let $R$ be an equivalence relation on $U$. For all $x\in U$, $\textbf{I}(M^{*}(R)/\{x\})= \textbf{I}(M^{*}(R)/RN(x))$.
\end{proposition}

\begin{proof}
According to Proposition~\ref{proposition4} and Corollary~\ref{corollary3}, we know $\{x\}\in \textbf{B}(M^{*}(R)|\{x\})$ and $\{x\}\in \textbf{B}(M^{*}(R)|RN(x))$.
Thus $\textbf{I}(M^{*}(R)/\{x\})=\{I\subseteq U-\{x\}:I\bigcup \{x\}\in \textbf{I}^{*}(R)\}$, $\textbf{I}(M^{*}(R)/RN(x))=$
$\{I\subseteq U-RN(x):I\bigcup \{x\}\in \textbf{I}^{*}(R)\}$.
For all $Y\subseteq RN(x)-\{x\}$ and $Y\neq \emptyset$, $Y\bigcup \{x\}\notin \textbf{I}^{*}(R)$. Thus $\textbf{I}(M^{*}(R)/\{x\})=\{I\subseteq U-\{x\}:I\bigcup \{x\}\in \textbf{I}^{*}(R)\}=\{I\subseteq U-RN(x):I\bigcup \{x\}\in $ $\textbf{I}^{*}(R)\}$. Hence $\textbf{I}(M^{*}(R)/\{x\})=\textbf{I}(M^{*}(R)/RN(x))$. This completes the proof.
\end{proof}

\begin{proposition}
\label{proposition9}
Let $R$ be an equivalence relation on $U$. For all $x\in U$, $\textbf{B}(M^{*}(R)/\{x\})=\textbf{B}(M^{*}(R)/RN(x))$.
\end{proposition}

\begin{proof}
According to Definition~\ref{definition6}, $\textbf{B}(M^{*}(R)/\{x\})=Max(\textbf{I}(M^{*}(R)/$ $\{x\}))$,
and $\textbf{B}(M^{*}(R)/RN(x))=Max(\textbf{I}(M^{*}(R)/RN(x)))$.
According to Proposition~\ref{proposition7}, $\textbf{I}(M^{*}(R)/\{x\})=\textbf{I}(M^{*}(R)/RN(x))$.
Thus $\textbf{B}(M^{*}(R)/$ $\{x\})=\textbf{B}(M^{*}(R)/RN(x))$. This completes the proof.
\end{proof}

\begin{proposition}
\label{proposition10}
Let $R$ be an equivalence relation on $U$. For all $x\in U$ and $X\subseteq U-RN(x)$, $r_{M^{*}(R)/\{x\}}(X)=r_{M^{*}(R)/RN(x)}(X)$.
\end{proposition}

\begin{proof}
For all $X\subseteq U-RN(x)$, $X\subseteq U-\{x\}$.
According to Proposition~\ref{proposition6}, $r_{M^{*}(R)/\{x\}}(X)=r_{M^{*}(R)}(X\bigcup \{x\})-r_{M^{*}(R)}(\{x\})$ and $r_{M^{*}(R)/RN(x)}$\\$(X)=r_{M^{*}(R)}(X\bigcup RN(x))-r_{M^{*}(R)}(RN(x))$.
According to Corollary~\ref{corollary4}, $r_{M^{*}(R)}(\{x\})=r_{M^{*}(R)}(RN(x))=1$.
So, we only need to proof $r_{M^{*}(R)}(X\bigcup \{x\})=r_{M^{*}(R)}(X\bigcup RN(x))$.
For all $y\in (X\bigcup RN(x))-(X\bigcup \{x\})=RN(x)-\{x\}$, $r_{M^{*}(R)}(X\bigcup \{x\})=r_{M^{*}(R)}(X\bigcup \{x,y\})$.
According to Proposition~\ref{proposition5}, $r_{M^{*}(R)}(X\bigcup \{x\})=r_{M^{*}(R)}((X\bigcup \{x\})\bigcup (X$ $\bigcup RN(x)))= r_{M^{*}(R)}(X\bigcup RN(x))$. This completes the proof.
\end{proof}

\begin{proposition}
\label{proposition8}
Let $R$ be an equivalence relation on $U$. For all $x\in U$, $\textbf{C}(M^{*}(R)/RN(x))\subseteq \textbf{C}(M^{*}(R)/\{x\})$.
\end{proposition}

\begin{proof}
According to Proposition~\ref{proposition7}, $\textbf{I}(M^{*}(R)/\{x\})=\textbf{I}(M^{*}(R)/RN(x))$.
According to Definition~\ref{definition5}, $\textbf{C}(M^{*}(R)/\{x\})=Min((\textbf{I}(M^{*}(R)/\{x\}))^{c})$, where $(\textbf{I}(M^{*}(R)/\{x\}))^{c}=2^{U-\{x\}}-\textbf{I}(M^{*}(R)/\{x\})$, and $\textbf{C}(M^{*}(R)/$ $RN(x))=Min((\textbf{I}(M^{*}(R)/RN(x)))^{c})$, where $(\textbf{I}(M^{*}(R)$ $/RN(x)))^{c}=2^{U-RN(x)}-\textbf{I}(M^{*}(R)/RN(x))$.
Therefore, $Min((\textbf{I}(M^{*}(R)/RN(x)))^{c})$ $\subseteq Min((\textbf{I}(M^{*}(R)/\{x\}))^{c})$, i.e., $\textbf{C}(M^{*}(R)/RN(x))\subseteq \textbf{C}(M^{*}(R)/\{x\})$.
This completes the proof.
\end{proof}

In the following proposition, we study when the contraction of the dual matroid to the complement of a single point set and the contraction of the dual matroid to the complement of the equivalence class of this point have the same circuits.

\begin{proposition}
\label{proposition11}
Let $R$ be an equivalence relation on $U$. For all $x\in U$, $\textbf{C}(M^{*}(R)/\{x\}|(U-RN(x)))=\textbf{C}(M^{*}(R)/RN(x))$.
\end{proposition}

\begin{proof}
According to Proposition~\ref{proposition7}, $\textbf{I}(M^{*}(R)/\{x\})=\textbf{I}(M^{*}(R)/RN(x))$ and $\textbf{I}(M^{*}(R)/\{x\})=$ $\textbf{I}(M^{*}(R)/\{x\}|(U-RN(x)))$.
According to Definition~\ref{definition5}, $\textbf{C}(M^{*}(R)/\{x\}|(U-RN(x)))=Min((\textbf{I}(M^{*}(R)/\{x\}|(U-RN(x))))^{c})$, where $(\textbf{I}(M^{*}(R)/\{x\}|(U-RN(x))))^{c}=2^{U-RN(x)}-\textbf{I}(M^{*}($ $R)/\{x\}|(U-RN(x)))=2^{U-RN(x)}-\textbf{I}(M^{*}($ $R)/\{x\})$, and $\textbf{C}(M^{*}(R)/$ $RN(x))=Min((\textbf{I}(M^{*}(R)/RN(x)))^{c})$, where $(\textbf{I}(M^{*}(R)/RN(x)))^{c}=2^{U-RN(x)}-\textbf{I}(M^{*}(R)/RN(x))$.
Hence, $\textbf{C}(M^{*}($ $R)/\{x\}|(U-RN(x)))=\textbf{C}(M^{*}(R)/RN(x))$.
This completes the proof.
\end{proof}

As we know, if a single point set is the equivalence class of this point, then the contraction of the dual matroid to the complement of the single point set and the contraction of the dual matroid to the complement of the equivalence class of this point have the same independent sets, circuits, bases and rank functions.

\section{Conclusions}
\label{section3}
\label{S:Conclusions}
In this paper, we establish a matroidal structure of rough sets by the lower approximation operator.
We provide the dual of the matroid and its properties such as independent sets, bases and rank function.
Moreover, we study the relationships between the contraction of the dual matroid to the complement of a single point set and the contraction of the dual matroid to the complement of the equivalence class of this point.
We will do more works in combining rough sets and matroids.

\section{Acknowledgments}
This work is supported in part by the National Natural Science Foundation of China under Grant No.
61170128, the Natural Science Foundation of Fujian Province, China, under Grant Nos.
2011J01374 and 2012J01294, and the Science and Technology Key Project of Fujian Prov-ince, China, under Grant No.
2012H0043.


%

\end{document}